\documentclass[twoside,11pt]{article}

\usepackage{jmlr2e}
\usepackage{graphicx}
\usepackage{algorithmic}
\usepackage{algorithm}
\usepackage{fancyhdr} 
\usepackage{color}
\usepackage{natbib}
\usepackage{comment}
\usepackage{url}
\usepackage{epsf}
\usepackage{epsfig}
\usepackage{multirow}
\usepackage{tikz}
\usepackage{mathtools}
\usepackage{bm}
\usepackage{amsmath, amssymb}
\usepackage{bbm}

\newcommand{\xdomain}{\ensuremath{\mathcal{X}}}
\renewcommand{\H}{\ensuremath{\mathcal{H}}}

\newcommand{\R}{\mathbb{R}}

\jmlrheading{1}{2000}{1-48}{4/00}{10/00}{Marina Meil\u{a} and Michael I. Jordan}

\ShortHeadings{}{}
\firstpageno{1}

\begin{document}

\title{On the boosting ability of top-down decision tree learning algorithm for multiclass classification}

\author{\name Anna Choromanska\thanks{Equal contribution.} \email achoroma@cims.nyu.edu \\
       \addr Courant Institue of Mathematical Sciences, NYU\\
       NY, USA
       \AND
       \name Krzysztof Choromanski$^{*}$ \email kchoro@google.com \\
       \addr Google Research New York\\
       NY, USA
       \AND
       \name Mariusz Bojarski \email mbojarski@nvidia.com \\
       \addr NVIDIA Corporation\\
       NJ, USA}


\maketitle

\begin{abstract}
We analyze the performance of the top-down multiclass classification algorithm for decision tree learning called LOMtree, recently proposed in the literature~\cite{DBLP:journals/corr/ChoromanskaL14} for solving efficiently classification problems with very large number of classes. The algorithm online optimizes the objective function which simultaneously controls the depth of the tree and its statistical accuracy. We prove important properties of this objective and explore its connection to three well-known entropy-based decision tree objectives, i.e. Shannon entropy, Gini-entropy and its modified version, for which instead online optimization schemes were not yet developed. We show, via boosting-type guarantees, that maximizing the considered objective leads also to the reduction of all of these entropy-based objectives. The bounds we obtain critically depend on the strong-concavity properties of the entropy-based criteria, where the mildest dependence on the number of classes (only logarithmic) corresponds to the Shannon entropy. 
\end{abstract}

\begin{keywords}
multiclass classification, decision trees, boosting, online learning
\end{keywords}

\section{Introduction}

This paper focuses on the multiclass classification problem with very large number of classes which becomes of increasing importance with the recent widespread development of data-acquisition web services and devices. Straightforward extensions of the binary approaches to the multiclass setting, such as one-against-all approach~\cite{Rifkin2004}, do not often work in the presence of strict computational constraints. In this case, hierarchical approaches seem particularly favorable since, due to their structure, they potentially can significantly reduce the computational costs. This paper is motivated by very recent advances in the area of multiclass classification, and considers a hierarchical approach for learning a multiclass decision tree structure in a \emph{top-down} fashion, where splitting the data in every node of the tree is based on the value of a very particular objective function. This objective function controls the balancedness of the splits (thus the depth of the tree) and the statistical error they induce (thus the statistical error of the tree), and was initially introduced in~\cite{DBLP:journals/corr/ChoromanskaL14} along with the algorithm optimizing it in an online fashion called LOMtree. The algorithm was empirically shown to obtain high-quality trees in logarithmic (in the label complexity) train and test running times, simultaneously outperforming state-of-the-art comparators, yet the objective underlying it is still not-well understood. The main contribution of this work is an extensive theoretical analysis of the properties of this objective, and the algorithm\footnote{We do not discuss the algorithm here, we refer the reader to the original paper.} optimizing it. And in particular, the analysis includes exploring, via the boosting framework, a relation of this objective to some more standard entropy-based decision tree objectives, i.e. Shannon entropy\footnote{Throughout the paper we refer to Shannon entropy as simply entropy.}, Gini-entropy and its modified version, for which online optimization schemes in the context of multiclass classification were not yet developed. 

The multiclass classification problem was relatively recently explored in the literature, and there exist only few works addressing the problem. In this work we focus on decision tree-based approaches. Filter tree~\cite{BeygelzimerLR09} considers simplified instance of the problem where the tree structure over the labels is assumed given. It is provably consistent and achieves regret bound which depends logarithmically on the number of classes. The conditional probability tree~\cite{BeygelzimerLLSS09} instead learns the tree structure and uses the node splitting criterion which compromises between obtaining balanced split in a tree node and violating the recommendation for the split from the node regressor. The authors also provide regret bounds which scales with the tree depth. Other works, which come with no guarantees, consider splitting the data in every tree node by optimizing efficiency with accuracy constraints allowing fine-grained control of the efficiency-accuracy tradeoff~\cite{DengSBL11}, or by performing clustering~\cite{BengioWG10,journals/informaticaSI/MadzarovGC09}. The splitting criterion (objective function) analyzed in this paper differs from the criteria considered in previous works and comes with much stronger theoretical justification given in Section~\ref{sec:obj}.

The main theoretical analysis of this paper is kept in the boosting framework~\cite{Schapire:2012:BFA:2207821} and relies on the assumption of the existence of \emph{weak learners} in the tree nodes, where the top-down algorithm we study will amplify this weak advantage to build a tree achieving any desired level of accuracy wrt. entropy-based criteria. We add new theoretical results to the theory of boosting for the multiclass classification problem (the multiclass boosting is largely ununderstood, we refer the reader to~\cite{journals/jmlr/MukherjeeS13} for comprehensive review), and we show that LOMtree is a boosting algorithm reducing standard entropy-based criteria, where the obtained bounds depend on the strong concativity properties of these criteria. Our work extends two previous works: it significantly adds to the theoretical analysis of~\cite{DBLP:journals/corr/ChoromanskaL14}, where only Shannon entropy is considered, in which case we also slightly improve their bound, and it extends beyond the boosting analysis of the binary case of~\cite{Kearns95}. The main theoretical results are presented in Section~\ref{sec:main}. Numerical experiments (Section~\ref{sec:exp}) and brief discussion (Section~\ref{sec:con}) conclude the paper. 

\section{Objective function and its theoretical properties}
\label{sec:obj}

In this section we describe the objective function that is of central interest to this paper, and we provide its theoretical properties.

\subsection{Objective function}
\label{sec:outline}

We receive examples $x \in \xdomain \subseteq
\R^d$, with labels $y \in \{1,2,\ldots, k\}$. We assume access to
a hypothesis class $\H$ where each $h \in \H$ is a binary classifier,
$h~:~\xdomain\mapsto \{-1, 1\}$, and each node in the tree consists of a
classifier from $\H$. The classifiers are trained in such a way that
$h_n(x) = 1$ ($h_n$ denotes the classifier in node $n$ of the tree; for fixed node $n$ we will refer to $h_n$ simply as $h$) means that the example $x$ is sent to the right subtree
of node $n$, while $h_n(x) = -1$ sends $x$ to the left subtree. When we
reach a leaf node, we predict according to the label with the
highest frequency amongst the examples reaching that leaf.

Notice that from the perspective of reducing the computational complexity, we want to encourage the
number of examples going to the left and right to be \emph{balanced}. Furthermore, for maintaining good statistical accuracy, we want to send examples
of class $i$ almost exclusively to either the left or the right
subtree. A measure of whether the examples of each class reaching the node are then mostly sent to its one child node (pure split) or otherwise to both children (impure split) is referred to as the \emph{purity} of a tree node. These two criteria, purity and balancedness, were discussed in~\cite{DBLP:journals/corr/ChoromanskaL14}. This work also proposes an objective (convex) expressing both criteria, and thus measuring the quality of a hypothesis $h \in \H$ in creating partitions at a fixed node $n$ in the tree. The objective is given as follows

\begin{equation}
  J(h) = 2\sum_{i=1}^k \pi_i \left| P(h(x) > 0) - P(h(x) > 0 | i)
  \right|,
  \label{eqn:objective}
\end{equation} 
where $\pi_i$ denotes the proportion of label $i$
amongst the examples reaching this node, $P(h(x) > 0)$ and $P(h(x)
> 0 | i)$ denote the fraction of examples reaching $n$ for which $h(x)
> 0$, marginally and conditional on class $i$ respectively. It was shown that this objective can be effectively maximized over hypotheses $h \in \H$, giving high-quality partitions, in an online fashion (recall that it remains unclear how to online optimize some of the more standard decision tree objectives such as entropy-based objectives). Despite that, this objective and its properties (including the relation to the more standard entropy-based objectives) remain not fully understood. Its exhaustive analysis is instead provided in this paper. 

\subsection{Theoretical properties of the objective function}
\label{sec:objective}

We first define the concept of \textit{balancedness} and \textit{purity} of the split which are crucial for providing the theoretical properties of the objective function under consideration in this paper.

\begin{definition}[Purity and balancedness,~\cite{DBLP:journals/corr/ChoromanskaL14}]
The hypothesis $h \in \mathcal{H}$ induces a pure split if
\[\alpha := \sum_{i=1}^{k}\pi_i\min(P(h(x) > 0|i), P(h(x)<0|i)) \leq
\delta, 
\]
where $\delta \in [0,0.5)$, and $\alpha$ is called the
  \emph{purity factor}.

The hypothesis $h \in \mathcal{H}$ induces a balanced split if
\[c \leq \underbrace{P(h(x) > 0)}_{=\beta} \leq 1 - c,
\]
where $c \in (0,0.5]$, and $\beta$ is called the \emph{balancing
    factor}.
\end{definition}
A partition is called \emph{maximally pure} if $\alpha
= 0$ (each class is sent exclusively to the left or to the
right). A partition is called \emph{maximally balanced} if $\beta = 0.5$ (equal number of examples are sent to the left and to the
right). 

Next we show the first theoretical property of the objective function $J(h)$. Lemma~\ref{lemma:maximal} contains a stronger statement than the one in the original paper~\cite{DBLP:journals/corr/ChoromanskaL14} (Lemma 2).
\begin{lemma}
  For any hypothesis $h~:~\xdomain \mapsto \{-1,1\}$, the objective
  $J(h)$ satisfies $J(h) \in
  [0,1]$. Furthermore, $h$ induces a
  maximally pure and balanced partition iff $J(h) = 1$.
  \label{lemma:maximal}
\end{lemma}

Lemma~\ref{lemma:maximal} characterizes the behavior of the objective $J(h)$ at the optimum where $J(h) = 1$. In practice however we do not expect to have hypotheses producing maximally pure and balanced splits, thus it is of importance to be able to show that larger values of the objective correspond simultaneously to more pure and more balanced splits. This statement would fully justify why it is desired to maximize $J(h)$. We next focus on showing this property. We start by showing that increasing the value of the objective leads to more balanced splits. 

\begin{lemma}
  For any hypothesis $h$, and any distribution over examples $(x,y)$ the balancing factor $\beta$ satisfies $\beta \in \left[0.5(1 - \sqrt{1 - J(h)}),0.5(1 + \sqrt{1 - J(h)})\right]$.
\label{lemma:obj-to-balance}
\end{lemma}

Thus the larger (closer to $1$) the value of $J(h)$ is, the narrower the interval from Lemma~\ref{lemma:obj-to-balance} is, leading to more balanced splits ($\beta$ closer to $0.5$). 

The next lemma, which we borrow from the literature, relates the balancing and purity factor, and it will be used to show that increasing the value of the objective function corresponds not only to more balanced splits, but also to more pure splits. 

\begin{lemma}[\cite{DBLP:journals/corr/ChoromanskaL14}]
  For any hypothesis $h$, and any distribution over examples $(x,y)$,
  the purity factor $\alpha$ and the balancing factor $\beta$ satisfy $\alpha \leq \min\left\{(2 - J(h))/4\beta - \beta, 0.5\right \}$.
\label{lemma:obj-to-balance-purity}
\end{lemma}

Recall that Lemma~\ref{lemma:obj-to-balance} shows that increasing the value of $J(h)$ leads to a more balanced split ($\beta$ closer to $0.5$). From this fact and Lemma~\ref{lemma:obj-to-balance-purity} it follows that increasing the value of $J(h)$ leads to the upper-bound on $\alpha$ being closer to $0$ which also corresponds to a more pure split. Thus maximizing the objective recovers more balanced and more pure splits.

\begin{figure}[t]
\center
\setlength\tabcolsep{0pt}
\begin{tabular}{cc}
\includegraphics[width = 0.5\textwidth]{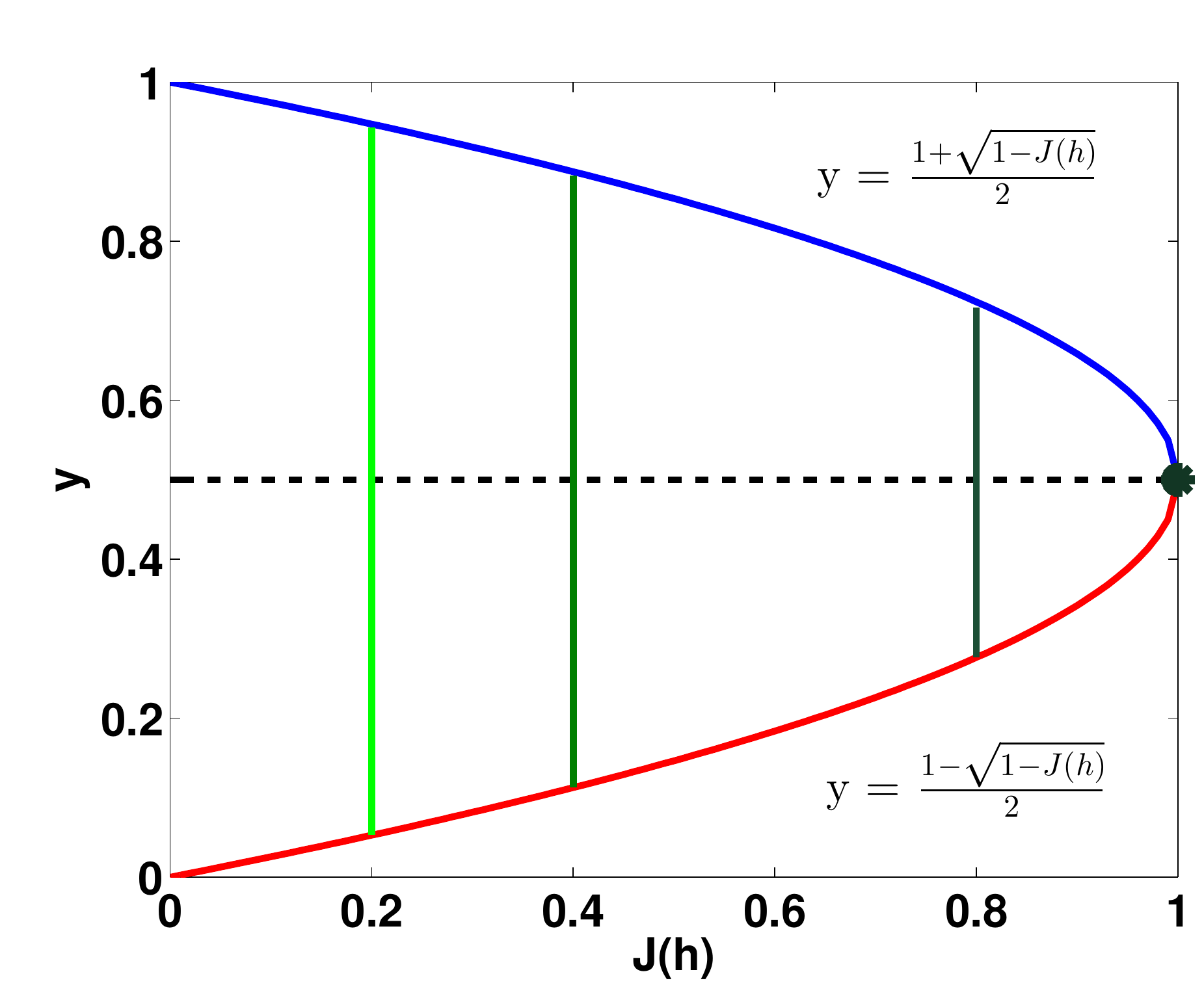}
\includegraphics[width = 0.5\textwidth]{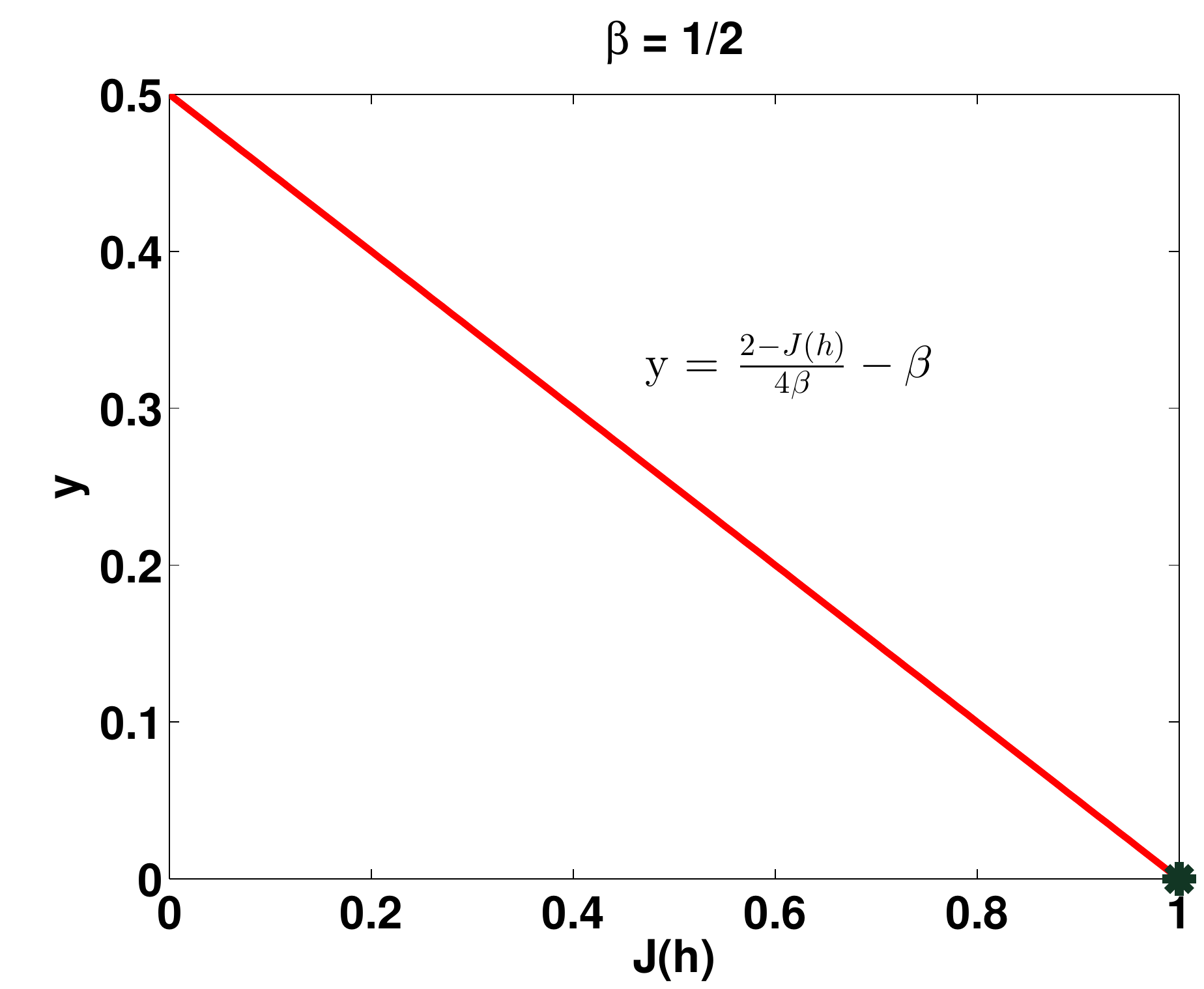}
\end{tabular}
\caption{\textbf{Left:} Blue curve captures the behavior of the upper-bound on the balancing factor as a function of $J(h)$, red curve captures the behavior of the lower-bound on the balancing factor as a function of $J(h)$, green intervals correspond to the intervals where the balancing factor lies for different values of $J(h)$. \textbf{Right:} Red line captures the behavior of the upper-bound on the purity factor as a function of $J(h)$ when the balancing factor is fixed to $\frac{1}{2}$. Figure should be read in color.}
\label{fig:balance_and_purity}
\end{figure}

\begin{proof}[Lemma~\ref{lemma:maximal}]
The proof that $J(h) \in [0,1]$ and that if $h$ induces a maximally pure and balanced partition then $J(h) = 1$ was done in~\cite{DBLP:journals/corr/ChoromanskaL14} (Lemma 2). We therefore prove here the remaining statement in Lemma~\ref{lemma:maximal} that if $J(h) = 1$ then $h$ induces a maximally pure and balanced partition.

Without loss of generality assume each $\pi_i \in (0,1)$. Recall that $\beta = P(h(x) > 0)$, and let $P_i = P(h(x) > 0 | i)$. Also recall that $\beta = \sum_{i=1}^k\pi_iP_i$. Thus $J(h) = 2\sum_{i=1}^k \pi_i \left| \sum_{j=1}^k\pi_jP_j - P_i\right|$. The objective is certainly maximized in the extremes of the interval $[0,1]$, where each $P_i$ is either $0$ or $1$ (also note that at maximum, where $J(h)=1$, it cannot be that all $P_i$'s are $0$ or all $P_i$'s are $1$). The function $J(h)$ is differentiable in these extremes ($J(h)$ is non-differentiable only when $\sum_{j=1}^k\pi_jP_j = P_i$, but at considered extremes the left-hand side of this equality is in $(0,1)$ whereas the right-hand side is either $0$ or $1$). We then write 
\[J(h) = 2\sum_{i\in \mathcal{P}} \pi_i \left(\sum_{j=1}^k\pi_jP_j - P_i\right) + 2\sum_{i\in \mathcal{N}} \pi_i \left(P_i - \sum_{j=1}^k\pi_jP_j\right), 
\]
where $\mathcal{P} = \{i:\sum_{j=1}^k\pi_jP_j \geq P_i\}$ and $\mathcal{N} = \{i:\sum_{j=1}^k\pi_jP_j < P_i\}$. Also let $\mathcal{P}^{+} = \{i:\sum_{j=1}^k\pi_jP_j > P_i\}$ (clearly $\sum_{i\in\mathcal{P}^{+}}\pi_i \neq 1$ and $\sum_{i\in\mathcal{N}}\pi_i \neq 1$ in the extremes of the interval $[0,1]$ where $J(h)$ is maximized). We then can compute the derivatives of $J(h)$ with respect to $P_r$, where $r = \{1,2,\dots,k\}$, everywhere where the function is differentiable as follows
\[\frac{\partial J}{\partial P_r} = \left \{
  \begin{tabular}{c}
  $2\pi_r(\sum_{i\in\mathcal{P}^{+}}\pi_i - 1)\:\:\:\:\:$ if$\:$$r\in\mathcal{P}^{+}$\\
  $\!\!\!2\pi_r(1 - \sum_{i\in\mathcal{N}}\pi_i)\:\:\:\:\:\:\:$ if$\:$$r\in\mathcal{N}$
  \end{tabular}
\right.,
\]
and note that in the extremes of the interval $[0,1]$ where $J(h)$ is maximized $\frac{\partial J}{\partial P_r} \neq 0$, since $\sum_{i\in\mathcal{P}^{+}}\pi_i \neq 1$, $\sum_{i\in\mathcal{N}}\pi_i \neq 1$, and each $\pi_i \in (0,1)$. Since $J(h)$ is convex, and by the fact that in particular the derivative of $J(h)$ with respect to any $P_r$ cannot be $0$ in the extremes of the interval $[0,1]$ where $J(h)$ is maximized, it follows that the $J(h)$ can only be maximized ($J(h) = 1$) at the extremes of the $[0,1]$ interval. Thus we already proved that if $J(h) = 1$ then $h$ induces a maximally pure partition. We are left with showing that if $J(h) = 1$ then $h$ induces also a maximally balanced partition. We prove it by contradiction. Assume $\beta \neq 0.5$. Denote as before $\mathcal{I}_0 = \{i:P(h(x) > 0 | i) = 0\}$ and $\mathcal{I}_1 = \{i:P(h(x) > 0 | i) = 1\}$. Recall $\beta = \sum_{i=1}^k\pi_iP_i = \sum_{i\in\mathcal{I}_0}\pi_i\cdot 0 + \sum_{i\in\mathcal{I}_1}\pi_i\cdot 1 = \sum_{i\in\mathcal{I}_1}\pi_i$.
Thus
\begin{eqnarray*}
J(h) &=& 1 = 2\sum_{i\in\mathcal{I}_0} \pi_i \left| \beta\right| + 2\sum_{i\in\mathcal{I}_1} \pi_i \left| \beta - 1\right| = 2\beta\sum_{i\in\mathcal{I}_0} \pi_i + 2(1-\beta)\sum_{i\in\mathcal{I}_1} \pi_i\\
&=& 2\beta(1\!-\!\!\sum_{i\in\mathcal{I}_1} \!\pi_i) + 2(1\!-\!\beta)\sum_{i\in\mathcal{I}_1} \!\pi_i = 2\beta(1\!-\!\beta) + 2(1\!-\!\beta)\beta = -4\beta^2+4\beta < 1,
\end{eqnarray*}
where the last inequality comes from the fact that the quadratic form $-4\beta^2+4\beta$ is equal to $1$ only when $\beta = 0.5$, and otherwise it is smaller than $1$. Thus we obtain the contradiction which ends the proof.
\end{proof}

\begin{proof}[Lemma~\ref{lemma:obj-to-balance}]
As before we use the following notation: $\beta = P(h(x) > 0)$, and $P_i = P(h(x) > 0 | i)$. Also let $\mathcal{P} = \{i:\beta \geq P_i\}$ and $\mathcal{N} = \{i:\beta < P_i\}$. Recall that $\beta = \sum_{i\in\{\mathcal{P}\cup \mathcal{N}\}} \pi_iP_i$, and $\sum_{i\in\{\mathcal{P}\cup \mathcal{N}\}} \pi_i = 1$. We split the proof into two cases.
\begin{itemize}
\item Let $\sum_{i\in\mathcal{P}}\pi_i \leq 1 - \beta$. Then
\begin{eqnarray*}
J(h) &=&  2\sum_{i=1}^k \pi_i \left| \beta - P_i\right| = 2\sum_{i\in\mathcal{P}} \pi_i (\beta - P_i) + 2\sum_{i\in\mathcal{N}} \pi_i (P_i - \beta)\\
&=& 2\sum_{i\in\mathcal{P}} \pi_i\beta - 2\sum_{i\in\mathcal{P}} \pi_iP_i + 2(\beta - \sum_{i\in\mathcal{P}} \pi_iP_i) - 2\beta(1 - \sum_{i\in\mathcal{P}}\pi_i)\\
&=& 4\beta\sum_{i\in\mathcal{P}}\pi_i - 4\sum_{i\in\mathcal{P}} \pi_iP_i \leq 4\beta\sum_{i\in\mathcal{P}}\pi_i \leq 4\beta(1-\beta)
\end{eqnarray*}
Thus $-4\beta^2+4\beta-J(h) \geq 0$ which, when solved, yields the lemma.
\item Let $\sum_{i\in\mathcal{P}}\pi_i \geq 1 - \beta$ (thus $\sum_{i\in\mathcal{N}}\pi_i \leq \beta$). Note that $J(h)$ can be written as
\[J(h) = 2\sum_{i=1}^k \pi_i \left| P(h(x) \leq 0) - P(h(x) \leq 0 | i)\right|,
\]
since $P(h(x) \leq 0) = 1 - P(h(x) > 0)$ and $P(h(x) \leq 0 | i) = 1 - P(h(x) > 0 | i)$. Let  $\beta^{'} = P(h(x) \leq 0) = 1 - \beta$, and $P_i^{'} = P(h(x) \leq 0 | i) = 1 - P_i$. Note that $\mathcal{P} = \{i:\beta \geq P_i\} = \{i:\beta^{'} < P_i^{'}\}$ and $\mathcal{N} = \{i:\beta < P_i\} = \{i:\beta^{'} \geq P_i^{'}\}$. Also note that $\beta^{'} = \sum_{i\in\{\mathcal{P}\cup \mathcal{N}\}} \pi_iP_i^{'}$. Thus
\begin{eqnarray*}
J(h) &=&  2\sum_{i=1}^k \pi_i \left| \beta^{'} - P_i^{'}\right| = 2\sum_{i\in\mathcal{P}} \pi_i (P_i^{'} - \beta^{'}) + 2\sum_{i\in\mathcal{N}} \pi_i (\beta^{'} - P_i^{'})\\
&=& 2(\beta^{'} - \sum_{i\in\mathcal{N}} \pi_iP_i^{'}) - 2\beta^{'}(1 - \sum_{i\in\mathcal{N}} \pi_i) + 2\sum_{i\in\mathcal{N}} \pi_i\beta^{'} - 2\sum_{i\in\mathcal{N}} \pi_iP_i^{'}\\
&=& 4\beta^{'}\sum_{i\in\mathcal{N}} \pi_i - 4\sum_{i\in\mathcal{N}} \pi_iP_i^{'} \leq 4\beta^{'}\sum_{i\in\mathcal{N}} \pi_i = 4(1-\beta)\sum_{i\in\mathcal{N}} \pi_i \leq  4\beta(1-\beta)
\end{eqnarray*}
Thus as before we obtain $-4\beta^2+4\beta-J(h) \geq 0$ which, when solved, yields the lemma.
\end{itemize}
\end{proof}

We next consider the quality of the entire tree as we add more nodes. We aim to maximize the objective function in each node we split. In the next section we show that optimizing the objective $J(h)$ leads to the reduction of the more standard decision tree entropy-based objectives. We consider three different objectives in this paper. We focus on the boosting framework, where the analysis depends on the \emph{weak learning assumption}. Three different entropy-based criteria lead to three different theoretical statements, where we bound the number of splits required to reduce the value of the criterion below given level. The bounds we obtain, and their dependence on $k$, critically depend on the strong concativity properties of the considered entropy-based criteria. In our analysis we use elements of the proof techniques from~\cite{Kearns95} (the proof of Theorem 10) and~\cite{DBLP:journals/corr/ChoromanskaL14} (the proof of Theorem 1). We show all the steps for completeness as we make modifications compared to these works. 

\section{Main theoretical results}
\label{sec:main}

We begin from explaining the notation. Let $\mathcal{T}$ denote the tree under consideration. $\pi_{l,i}$'s denote the probabilities that a randomly chosen data point $x$
drawn from $\mathcal{P}$, where $\mathcal{P}$ is a fixed target distribution over $\mathcal{X}$, has label
$i$ given that $x$ reaches node $l$ (note that $\sum_{i=1}^{k}\pi_{l,i} = 1$), $\mathcal{L}$ denotes the set of all tree leaves, $t$ denotes the number of internal tree nodes, and $w_l$ is the weight of leaf $l$ defined as the probability a randomly chosen $x$ drawn from
$\mathcal{P}$ reaches leaf $l$ (note that $\sum_{l \in \mathcal{L}}w_l = 1$). We study a tree construction algorithm where we recursively find the leaf node with the highest weight, and choose to
split it into two children. Consider the tree constructed over $t$
steps where in each step we take one leaf node and split it into two ($t = 1$ corresponds to splitting the root, thus the tree consists of one node (root) and its two children (leaves) in this step). Let $n$ be the
heaviest node at time $t$ and its weight $w_n$ be denoted by $w$ for
brevity. We measure the quality of the tree at any given time $t$ using three different entropy-based criteria:
\begin{itemize}
\item The entropy function $G_t^e$: $\:\:\:\:\:\:\:\:\:\:\:\:\:\:\:G_t^e = \sum_{l \in \mathcal{L}}w_l\sum_{i = 1}^k \pi_{l,i}\ln \left( \frac{1}{\pi_{l,i}} \right)$
\item The Gini-entropy function $G_t^g$: $\:\:\:\:\:\:G_t^g = \sum_{l \in \mathcal{L}}w_l\sum_{i = 1}^k \pi_{l,i}(1-\pi_{l,i})$
\item The modified Gini-entropy $G_t^m$: $\:\:\:\:G_t^m = \sum_{l \in \mathcal{L}}w_l\sum_{i = 1}^k \sqrt{\pi_{l,i}(\mathcal{C}-\pi_{l,i})},$\\
where $\mathcal{C}$ is a constant such that $\mathcal{C} > 2$. \\
\end{itemize}
These criteria are natural extensions of the criteria used in~\cite{Kearns95} in the context of binary classification, to the multiclass classification setting\footnote{Note that there is more than one way of extending the entropy-based criteria from~\cite{Kearns95} to the multiclass classification setting, e.g. the modified Gini-entropy could as well be defined as $G_t^m = \sum_{l \in \mathcal{L}}w_l\sum_{i = 1}^k \sqrt{\pi_{l,i}(\mathcal{C}-\pi_{l,i})}$ where $\mathcal{C}\in[1,2]$. This and other extensions will be investigated in future works.}. We will next present the main results of this paper which will be followed by their proofs. We begin with introducing the \emph{weak hypothesis assumption}.

Our theoretical analysis is held in the boosting framework and critically depends on the \emph{weak hypothesis assumption}, which ensures that the hypothesis class $\mathcal{H}$ is rich enough to guarantee 'weakly' pure and 'weakly' balanced split in any given node.

\begin{definition}[Weak Hypothesis Assumption,~\cite{DBLP:journals/corr/ChoromanskaL14}]
Let $m$ denotes any node of the tree
$\mathcal{T}$, and let $\beta_m = P(h_m(x) > 0)$ and $P_{m,i} = P(h_m(x) > 0|i)$. Furthermore, let $\gamma \in \mathbb{R}^{+}$ be such that for all $m$, $\gamma \in (0,\min(\beta_m,1-\beta_m)]$. We say that the \emph{weak hypothesis assumption} is satisfied when for any distribution
$\mathcal{P}$ over $\mathcal{X}$ at each node $m$ of the tree
$\mathcal{T}$ there exists a hypothesis $h_m \in \mathcal{H}$ such that $J(h_m)/2 =
\sum_{i = 1}^k \pi_{m,i}|P_{m,i} - \beta_{m}| \geq \gamma$.
\label{def:wha}
\end{definition}

We next state the three main theoretical results of this paper captured in Theorem~\ref{thm:main1},~\ref{thm:main2}, and~\ref{thm:main3}. They guarantee that the top-down decision tree algorithm which optimizes $J(h)$, such as the one in~\cite{DBLP:journals/corr/ChoromanskaL14}, will amplify the weak advantage, captured in the \emph{weak learning assumption}, to build a tree achieving any desired level of accuracy wrt. entropy-based criteria.

\begin{theorem}
Under the Weak Hypothesis Assumption, for any $\alpha \in [0,2\ln k]$, to
obtain $G_t^e \leq \alpha$ it suffices to make
\[t \geq \left(\frac{2\ln k}{\alpha}\right)^{\frac{4(1 - \gamma)^2}{\gamma^2\log_2 e}\ln
  k} \:\:\:\:\:\:\:\:\:\:\text{splits.}\footnote{This guarantee is slightly tighter compared to Theorem 1 in~\cite{DBLP:journals/corr/ChoromanskaL14}.}
\]
\label{thm:main1}
\end{theorem}

\begin{theorem}
Under the Weak Hypothesis Assumption, for any $\alpha \in [0,2\left(1-\frac{1}{k}\right)]$, to
obtain $G_t^g \leq \alpha$ it suffices to make
\[t \geq \left(\frac{2\left(1-\frac{1}{k}\right)}{\alpha}\right)^{\frac{2(1 - \gamma)^2}{\gamma^2\log_2 e}(k-1)} \:\:\:\:\:\:\:\:\:\:\text{splits.}
\]
\label{thm:main2}
\end{theorem}

\begin{theorem}
Under the Weak Hypothesis Assumption, for any $\alpha \in [\sqrt{\mathcal{C}-1},2\sqrt{k\mathcal{C}-1}]$, to
obtain $G_t^m \leq \alpha$ it suffices to make
\[t \geq \left(\frac{2\sqrt{k\mathcal{C}-1}}{\alpha}\right)^{\frac{2(1 - \gamma)^2\mathcal{C}^3}{\gamma^2(\mathcal{C}-2)^2\log_2 e}k\sqrt{k\mathcal{C}-1}} \:\:\:\:\:\:\:\:\:\:\text{splits.}
\]
\label{thm:main3}
\end{theorem}

Clearly, different criteria lead to bounds with different dependence on the number of classes $k$, where the most advantageous dependence (only logarithmic in $k$) is obtained for the entropy criterion. This is a consequence of the strong concativity properties of the entropy-based criteria considered in this paper. We next discuss in details the mathematical properties of the entropy-based criteria, which are important to prove the above theorems.

\subsection{Properties of the entropy-based criteria}

Each of the presented entropy-based criteria has a number of useful properties that we give next with their proofs.

\paragraph{Bounds on the entropy-based criteria} 

We first give bounds on the values of the entropy-based functions.

\begin{lemma}
The entropy function $G_t^e$ at time $t$ is bounded as 
\[0 \leq G_t^e \leq (t+1)w\ln k.
\]
\label{lem:bound_entropy}
\end{lemma}

\begin{proof}
The lower-bound follows from the fact that the entropy of each leaf $\sum_{i = 1}^k \pi_{l,i}\ln \left( \frac{1}{\pi_{l,i}} \right)$ is non-negative. We next prove the upper-bound.
Note that 
\[G_t^e = \sum_{l \in \mathcal{L}}w_l\sum_{i = 1}^k \pi_{l,i}\ln \left( \frac{1}{\pi_{l,i}} \right) \leq \sum_{l \in \mathcal{L}}w_l \ln k \leq w\ln k\sum_{l \in \mathcal{L}}1 = (t+1)w\ln k,
\]

where the first inequality comes from the fact that uniform distribution maximizes the entropy, and the last equality comes from the fact that a tree with $t$ internal nodes has $t+1$ leaves (also recall that $w$ is the weight of the heaviest node in the tree at time $t$ which is what we will also use in the next lemmas).
\end{proof}

Before proceeding to the Gini-entropy criterion we first introduce the helpful result captured in Lemma~\ref{lem:helpful1} and Corollary~\ref{corr:helpful2}.
\begin{lemma}[The inequality between Euclidean and arithmetic mean]
Let $x_1,x_2,\dots,x_k$ be a set of non-negative numbers. Then Euclidean mean upper-bounds the arithmetic mean as follows $\sqrt{\frac{\sum_{i=1}^kx_i^2}{k}} \geq \frac{\sum_{i=1}^kx_i}{k}$.
\label{lem:helpful1}
\end{lemma}

\begin{corollary}
Let $\{x_1,x_2,\dots,x_k\}$ be a set of non-negative numbers. Then $\sum_{i=1}^kx_i^2 \geq \frac{1}{k}\left(\sum_{i=1}^kx_i\right)^2$.
\label{corr:helpful2}
\end{corollary}

\begin{proof}
By Lemma~\ref{lem:helpful1} we have $\sqrt{\frac{\sum_{i=1}^kx_i^2}{k}} \geq \frac{\sum_{i=1}^kx_i}{k} \Leftrightarrow \sum_{i=1}^kx_i^2 \geq \frac{1}{k}\left(\sum_{i=1}^kx_i\right)^2$.
\end{proof}

We next proceed to the Gini-entropy.
\begin{lemma}
The Gini-entropy function $G_t^g$ at time $t$ is bounded as 
\[0 \leq G_t^g \leq (t+1)w\left(1-\frac{1}{k}\right).
\]
\label{lem:bound_gentropy}
\end{lemma}

\begin{proof}
The lower-bound is straightforward since all $\pi_{l,i}$'s are non-negative. The upper-bound can be shown as follows (the last inequality results from Corollary~\ref{corr:helpful2}):
\begin{eqnarray*}
G_t^g &=& \sum_{l \in \mathcal{L}}w_l\sum_{i = 1}^k \pi_{l,i}(1-\pi_{l,i}) \leq w\sum_{l \in \mathcal{L}}\sum_{i = 1}^k (\pi_{l,i}-\pi_{l,i}^2) = w\sum_{l \in \mathcal{L}}\left(1-\sum_{i = 1}^k\pi_{l,i}^2\right)\\ 
&\leq& w\sum_{l \in \mathcal{L}}\left(1-\frac{1}{k}\left(\sum_{i = 1}^k\pi_{l,i}\right)^2\right) = w\sum_{l \in \mathcal{L}}\left(1-\frac{1}{k}\right) = (t+1)w\left(1-\frac{1}{k}\right),
\end{eqnarray*}
\end{proof}

\begin{lemma}
The modified Gini-entropy function $G_t^m$ at time $t$ is bounded as
\[\sqrt{\mathcal{C}-1}\leq G_t^m \leq (t+1)w\sqrt{k\mathcal{C} - 1}.
\]
\label{lem:bound_mgentropy}
\end{lemma}

\begin{proof}
The lower-bound can be shown as follows. Recall that the function\\ 
$\sum_{i = 1}^k \sqrt{\pi_{l,i}(\mathcal{C}-\pi_{l,i})}$ is concave and therefore it is certainly minimized on the extremes of the $[0,1]$ interval, meaning where each $\pi_{l,i}$ is either $0$ or $1$. Let $I_0 = \{i:\pi_{l,i} = 0\}$ and let $I_1 = \{i:\pi_{l,i} = 1\}$. Thus $\sum_{i = 1}^k \sqrt{\pi_{l,i}(\mathcal{C}-\pi_{l,i})} = \sum_{i \in I_1} \sqrt{\mathcal{C}-1} \geq \sqrt{\mathcal{C}-1}$. Combining this result with the fact that $\sum_{l \in \mathcal{L}}w_l=1$ gives the lower-bound.
We next prove the upper-bound. Recall that from Lemma~\ref{lem:helpful1} it follows that
\[\frac{\sum_{i=1}^k\sqrt{\pi_{l,i}(\mathcal{C} - \pi_{l,i})}}{k} \leq \sqrt{\frac{\sum_{i=1}^k\pi_{l,i}(\mathcal{C} - \pi_{l,i})}{k}}, \:\:\:\:\:\text{thus}
\]

\begin{eqnarray*}
G_t^m &=& \sum_{l \in \mathcal{L}}w_l\sum_{i = 1}^k \sqrt{\pi_{l,i}(\mathcal{C}-\pi_{l,i})} \leq \sum_{l \in \mathcal{L}}w_l\sqrt{k\sum_{i=1}^k\pi_{l,i}(\mathcal{C} - \pi_{l,i})}\\
&=& \sum_{l \in \mathcal{L}}w_l\sqrt{k\left(\mathcal{C} - \sum_{i=1}^k\pi_{l,i}^2\right)} = \sum_{l \in \mathcal{L}}w_l\sqrt{k\mathcal{C} - k^2\sum_{i=1}^k\frac{1}{k}\pi_{l,i}^2}
\end{eqnarray*}
By Jensen's inequlity $\sum_{i=1}^k\frac{1}{k}\pi_{l,i}^2 \geq (\sum_{i=1}^k\frac{1}{k}\pi_{l,i})^2 = \frac{1}{k^2}$. Thus
\[G_t^m \leq \sum_{l \in \mathcal{L}}w_l\sqrt{k\mathcal{C} - 1} \leq (t+1)w\sqrt{k\mathcal{C} - 1}
\]
\end{proof}

So far we have been focusing on the time step $t$, where $n$ was the heaviest node and it had weight $w$. Consider splitting this leaf to two children $n_0$ and $n_1$. For the ease of notation let $w_0 = w_{n_0}$ and $w_1 =
w_{n_1}$, $\beta = P(h_n(x) > 0)$
and $P_i = P(h_n(x) > 0|i)$, and furthermore let $\pi_i$ and $h$ be the shorthands for
$\pi_{n,i}$ and $h_n$ respectively. Recall that $\beta = \sum_{i=1}^k\pi_iP_i$ and
$\sum_{i=1}^k\pi_i = 1$. Notice that $w_0 = w(1-\beta)$ and $w_1
= w\beta$. Let ${\bm \pi}$ be the $k$-element vector with $i^{th}$
entry equal to $\pi_i$. Finally, let $\tilde{G}^e({\bm \pi}) = \sum_{i =
  1}^k \pi_{i}\ln \left( \frac{1}{\pi_{i}} \right)$, $\tilde{G}^g({\bm \pi}) = \sum_{i =
  1}^k \pi_{i}(1-\pi_i)$, and $\tilde{G}^m({\bm \pi}) = \sum_{i =
  1}^k \sqrt{\pi_{i}(1-\pi_i)}$. Before the split the contribution of node $n$ to resp.
$G_t^e$, $G_t^g$, and $G_t^m$ was resp. $w\tilde{G}^e({\bm \pi})$, $w\tilde{G}^g({\bm \pi})$, and $w\tilde{G}^m({\bm \pi})$. Note that $\pi_{n_0,i} =
\frac{\pi_i(1 - P_i)}{1 - \beta}$ and $\pi_{n_1,i} =
\frac{\pi_iP_i}{\beta}$ are the probabilities that a randomly chosen $x$
drawn from $\mathcal{P}$ has label $i$ given that $x$ reaches nodes
$n_0$ and $n_1$ respectively. For brevity, let $\pi_{n_0,i}$ and $\pi_{n_1,i}$ be denoted respectively as $\pi_{0,i}$ and $\pi_{1,i}$. Let ${\bm \pi}_0$ be the
$k$-element vector with $i^{th}$ entry equal to $\pi_{0,i}$ and let
${\bm \pi}_1$ be the $k$-element vector with $i^{th}$ entry equal
to $\pi_{1,i}$. Notice that ${\bm \pi} = (1 - \beta){\bm \pi}_0 +
\beta{\bm \pi}_1$. After the split the contribution of the same,
now internal, node $n$ changes to resp. $w((1- \beta)\tilde{G}^e({\bm \pi}_0) + \beta
\tilde{G}^e({\bm \pi}_1))$,  $w((1- \beta)\tilde{G}^g({\bm \pi}_0) + \beta
\tilde{G}^g({\bm \pi}_1))$, and  $w((1- \beta)\tilde{G}^m({\bm \pi}_0) + \beta
\tilde{G}^m({\bm \pi}_1))$. We denote the difference between the contribution of node $n$ to the value of the entropy-based objectives in times $t$ and $t+1$ as
\begin{equation}
  \Delta_t^e := G_t^e - G_{t+1}^e =  w\left[\tilde{G}^e({\bm \pi}) - (1- \beta)\tilde{G}^e({\bm
    \pi}_0) - \beta \tilde{G}^e({\bm \pi}_1)\right].
  \label{eqn:ent-decrease}
\end{equation}
\begin{equation}
  \Delta_t^g := G_t^g - G_{t+1}^g =  w\left[\tilde{G}^g({\bm \pi}) - (1- \beta)\tilde{G}^g({\bm
    \pi}_0) - \beta \tilde{G}^g({\bm \pi}_1)\right].
  \label{eqn:gent-decrease}
\end{equation}
\begin{equation}
  \Delta_t^m := G_t^m - G_{t+1}^m =  w\left[\tilde{G}^m({\bm \pi}) - (1- \beta)\tilde{G}^m({\bm
    \pi}_0) - \beta \tilde{G}^m({\bm \pi}_1)\right].
  \label{eqn:mgent-decrease}
\end{equation}

\paragraph{Strong concativity properties of the entropy-based criteria} 

The next three lemmas, Lemma~\ref{lem:sc_entropy},~\ref{lem:sc_gentropy}, and~\ref{lem:sc_mgentropy}, describe the strong concativity properties of the entropy, Gini-entropy and modified Gini-entropy which can be used to lower-bound $\Delta_t^e$, $\Delta_t^g$, and $\Delta_t^m$ (Equations~\ref{eqn:ent-decrease},~\ref{eqn:gent-decrease}, and~\ref{eqn:mgent-decrease} corresponds to a gap in the Jensen's inequality applied to the strongly concave function).
\begin{lemma}
The entropy function $\tilde{G}^e$ is strongly concave  with respect to $l_1$-norm with modulus $1$, and thus the following holds
\[\tilde{G}^e({\bm \pi}) - (1- \beta)\tilde{G}^e({\bm\pi}_0) - \beta \tilde{G}^e({\bm \pi}_1) \geq \frac{1}{2}\beta(1-\beta)\|\bm\pi_0 - \bm\pi_1\|_1^2.
\]
\label{lem:sc_entropy}
\end{lemma}

\begin{proof}
Lemma~\ref{lem:sc_entropy} is proven in~\cite{ShaiSS2012} (Example 2.5).
\end{proof}

We introduce one more lemma and then proceed with Gini-entropy.
\begin{lemma}[\cite{Shalev-ShwartzThesis2007}(Lemma 14)]
If the function $\Phi(\bm \pi)$ is twice differentiable,  then the sufficient condition for strong concativity of $\Phi$ is that for all $\bm \pi$, $\bm x$, $\left<\nabla^2\Phi(\bm\pi)\bm x,\bm x\right> \leq -\sigma\|x\|^2$, where $\nabla^2\Phi(\bm\pi)$  is the Hessian matrix of $\Phi$ at $\bm \pi$, and $\sigma > 0$ is the strong concativity modulus.
\label{lem:sss}
\end{lemma}

\begin{lemma}
The Gini-entropy function $\tilde{G}^g$ is strongly concave with respect to $l_2$-norm with modulus $2$, and thus the following holds
\[\tilde{G}^g({\bm \pi}) - (1- \beta)\tilde{G}^g({\bm\pi}_0) - \beta \tilde{G}^g({\bm \pi}_1) \geq \beta(1-\beta)\|\bm\pi_0 - \bm\pi_1\|_2^2.
\]
\label{lem:sc_gentropy}
\end{lemma}

\begin{proof}
Note that $\left<\nabla^2\tilde{G}^g(\bm\pi)\bm x,\bm x\right>\leq -2\|\bm x\|_2^2$, and apply Lemma~\ref{lem:sss}.
\end{proof}

Before showing the strong concativity guarantee for the modified Gini-entropy, we first show the statement that will be useful to prove the lemma.

\begin{lemma}[\cite{zhukovskiy2003lyapunov}, Remark 2.2.4.]
The sum of strongly concave functions on $\mathbb{R}^n$ with modulus $\sigma$ is strongly concave with the same modulus.
\label{lem:concuseful}
\end{lemma}  

\begin{lemma}
The modified Gini-entropy function $\tilde{G}^m$ is strongly concave with respect to $l_2$-norm with modulus $\frac{2(\mathcal{C}-2)^2}{\mathcal{C}^3}$, and thus the following holds
\[\tilde{G}^m({\bm \pi}) - (1- \beta)\tilde{G}^m({\bm\pi}_0) - \beta \tilde{G}^m({\bm \pi}_1) \geq \frac{(\mathcal{C}-2)^2}{\mathcal{C}^3}\beta(1-\beta)\|\bm\pi_0 - \bm\pi_1\|_2^2.
\]
\label{lem:sc_mgentropy}
\end{lemma}

\begin{figure}[t]
\center
\setlength\tabcolsep{0pt}
\begin{tabular}{cc}
\includegraphics[width = 0.5\textwidth]{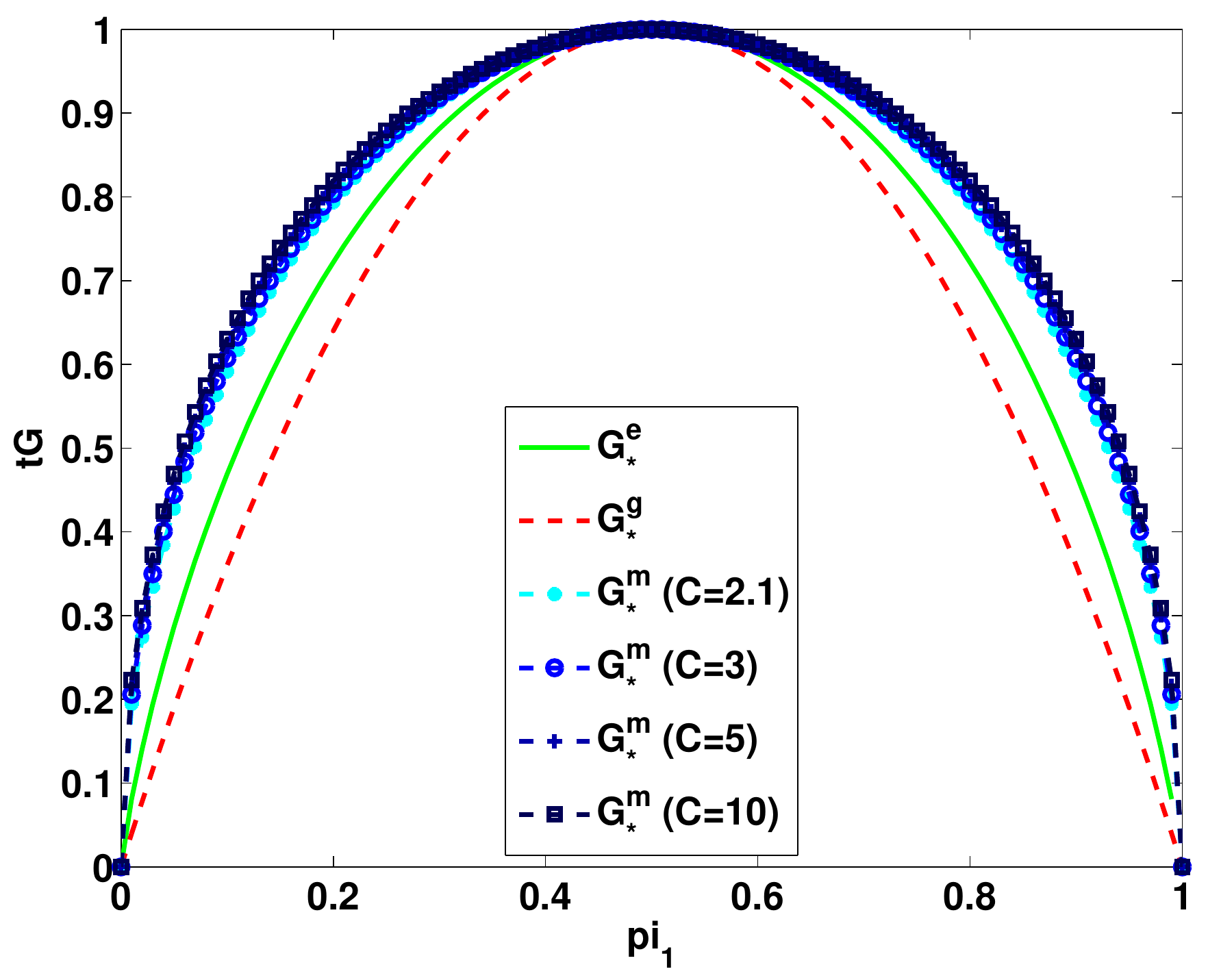}
\includegraphics[width = 0.5\textwidth]{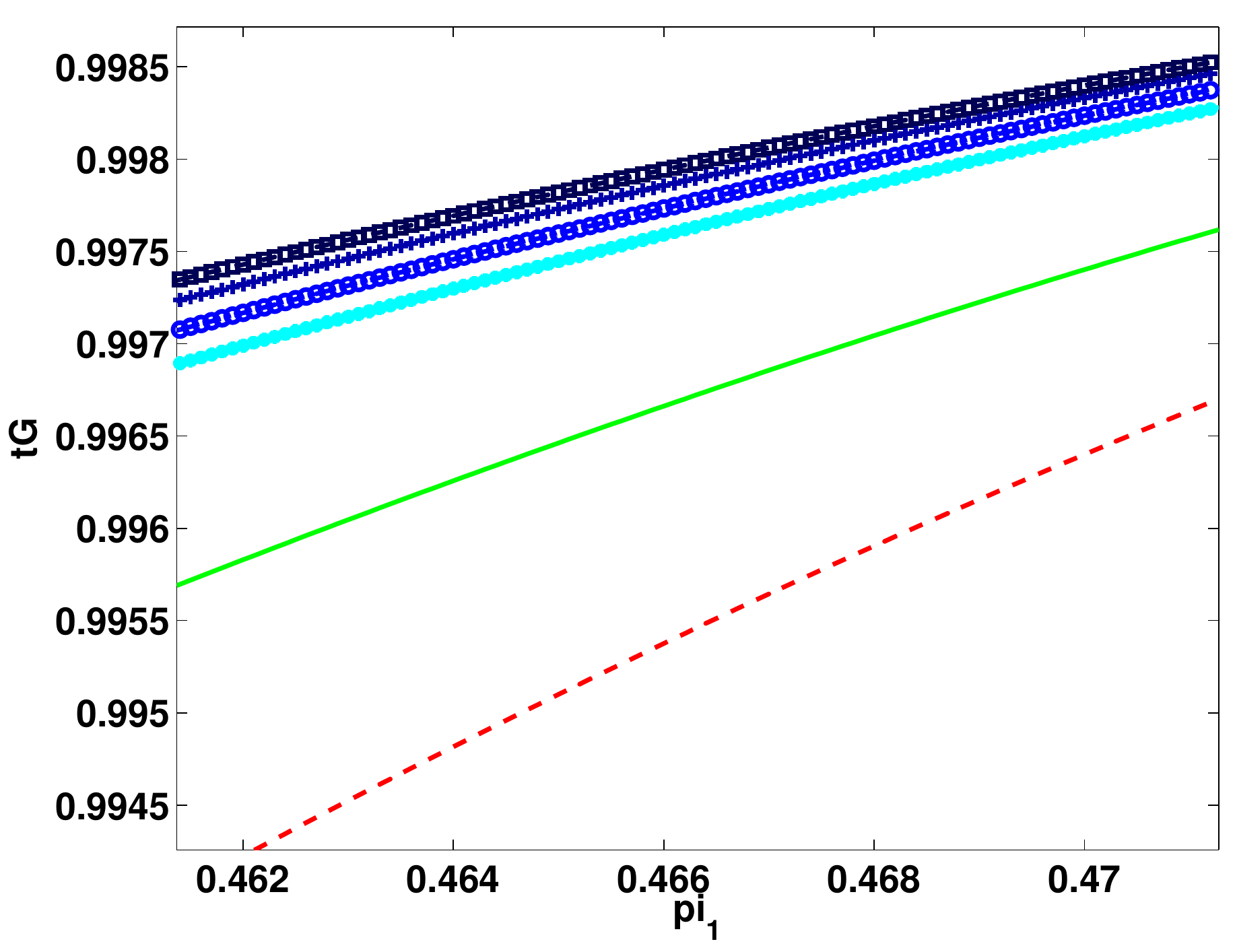}
\end{tabular}
\caption{Functions $G^e_{*}(\pi_1) = \tilde{G}^e(\pi_1)/\ln 2 = \left(\pi_1\ln\left(\frac{1}{\pi_1}\right) +(1-\pi_1)\ln\left(\frac{1}{1-\pi_1}\right)\right)/\ln 2$, $G^g_{*}(\pi_1) = 2\tilde{G}^g(\pi_1) = 4\pi_1(1-\pi_1)$, and $G^m_{*}(\pi_1) = (\tilde{G}^m(\pi_1) - \sqrt{\mathcal{C}-1})/(\sqrt{2*\mathcal{C}-1} - \sqrt{\mathcal{C}-1)} = (\sqrt{\pi_1(\mathcal{C} - \pi_1)} + \sqrt{(1-\pi_1)(\mathcal{C} - 1 + \pi_1)} - \sqrt{\mathcal{C}-1})/(\sqrt{2*\mathcal{C}-1} - \sqrt{\mathcal{C}-1)}$ (functions $\tilde{G}^e(\pi_1)$, $\tilde{G}^g(\pi_1)$, and $\tilde{G}^m(\pi_1)$ were re-scaled to have values in $[0,1]$) as a function of $\pi_1$ ($pi_1$). Figure is recommended to be read in color.}
\label{fig:balance_and_purity}
\end{figure}

\begin{proof}
Consider functions $g(\pi_i) = \sqrt{f(\pi_i)}$, where $f(\pi_i) = \pi_i(\mathcal{C}-\pi_i)$, $\mathcal{C} \geq 2$, and $\pi_i \in [0,1]$. Also let $h(x) = \sqrt{x}$, where $x \in [0,\frac{\mathcal{C}^2}{4}]$. It is easy to see, using Lemma~\ref{lem:sss}, that function $f$ is strongly concave with respect to $l_2$-norm with modulus $2$, thus
\begin{equation}
f(\theta \pi_i^{'} + (1-\theta)\pi_i^{''}) \geq \theta f(\pi_i^{'}) + (1-\theta)f(\pi_i^{''}) + \theta(1-\theta)\|\pi_i^{'} - \pi_i^{''}\|_2^2,
\label{eq:firststep}
\end{equation}
where $\pi_i^{'},\pi_i^{''} \in [0,1]$ and $\theta \in [0,1]$. Also note that $h$ is strongly concave with modulus $\frac{2}{\mathcal{C}^3}$ in its domain $[0,\frac{\mathcal{C}^2}{4}]$ (the second derivative of $h$ is $h^{''}(x) = -\frac{1}{4\sqrt{x^3}} \leq -\frac{2}{\mathcal{C}^3}$).  The strong concativity of $h$ implies that
\[\sqrt{\theta x_1+(1-\theta)x_2} \geq \theta \sqrt{x_1} + (1-\theta)\sqrt{x_2} + \frac{1}{\mathcal{C}^3}\theta(1-\theta)\|x_1-x_2\|_2^2,
\]
where $x_1,x_2 \in [0,\frac{\mathcal{C}^2}{4}]$. Let $x_1 = f(\pi_i^{'})$ and $x_2 = f(\pi_i^{''})$. Then we obtain
\begin{equation}
\sqrt{\theta f(\pi_i^{'})+(1-\theta)f(\pi_i^{''})} \geq \theta \sqrt{f(\pi_i^{'})} + (1-\theta)\sqrt{f(\pi_i^{''})} + \frac{1}{\mathcal{C}^3}\theta(1-\theta)\|f(\pi_i^{'})-f(\pi_i^{''})\|_2^2. 
\label{eq:secondstep}
\end{equation}
Note that
\begin{eqnarray*}
\sqrt{f(\theta \pi_i^{'} + (1-\theta)\pi_i^{''})} &\geq& \sqrt{f(\theta \pi_i^{'} + (1-\theta)\pi_i^{''}) - \theta(1-\theta)\|\pi_i^{'}-\pi_i^{''}\|_2^2}\\
&\geq& \sqrt{\theta f(\pi_i^{'}) + (1-\theta)f(\pi_i^{''})}\\
&\geq& \theta \sqrt{f(\pi_i^{'})} + (1-\theta)\sqrt{f(\pi_i^{''})} + \frac{1}{\mathcal{C}^3}\theta(1-\theta)\|f(\pi_i^{'})-f(\pi_i^{''})\|_2^2,
\end{eqnarray*}
where the second inequality results from Equation~\ref{eq:firststep} and the last (third) inequality results from Equation~\ref{eq:secondstep}. Finally note that the first derivative of $f$ is $f^{'}(\pi_i) = \mathcal{C} - 2\pi_i \in [\mathcal{C}-2,\mathcal{C}]$ thus
\[\frac{|f(\pi_i^{'})-f(\pi_i^{''})|}{|\pi_i^{'} - \pi_i^{''}|} \geq \mathcal{C}-2 \Leftrightarrow \|f(\pi_i^{'})-f(\pi_i^{''})\|^2 \geq (\mathcal{C}-2)^2\|\pi_i^{'} - \pi_i^{''}\|^2,
\]
and combining this result with previous statement yields
\[\sqrt{f(\theta \pi_i^{'} + (1-\theta)\pi_i^{''})} \geq \theta \sqrt{f(\pi_i^{'})} + (1-\theta)\sqrt{f(\pi_i^{''})} + \frac{(\mathcal{C}-2)^2}{\mathcal{C}^3}\theta(1-\theta)\|\pi_i^{'} - \pi_i^{''}\|^2,
\]
thus $g(\pi_i)$ is strongly concave with modulus $\frac{2(\mathcal{C}-2)^2}{\mathcal{C}^3}$. By Lemma~\ref{lem:concuseful}, $\tilde{G}^m(\bm \pi)$ is also strongly concave with the same modulus.
\end{proof}

\subsection{Proof of the main theorems}

We finally proceed to proving all three theorems. We first introduce some mathematical tools that will be used in the following proofs. The next two lemma are fundamental. The first one relates $l_1$-norm and $l_2$-norm and the second one is a simple property of the exponential function.
\begin{lemma}
Let $x \in \mathbb{R}^{k}$ then $\|x\|_1 \leq \sqrt{k}\|x\|_2$.
\label{lem:norms}
\end{lemma}

\begin{lemma}
For $x \geq 1$ the following holds $\left(1 - \frac{1}{x}\right)^x \leq \frac{1}{e}$.
\label{lem:exp}
\end{lemma}

We next proceed to proving Theorem~\ref{thm:main1},~\ref{thm:main2}, and~\ref{thm:main3}.

\begin{proof}
For the entropy it follows from Equation~\ref{eqn:ent-decrease} and Lemma~\ref{lem:sc_entropy} that
\begin{eqnarray}
\Delta_t^e &\geq& \frac{1}{2}w\beta(1-\beta)\|\bm \pi_0 - \bm \pi_1\|_1^2 = \frac{1}{2}\frac{w}{\beta(1-\beta)}\left(\sum_{i
  = 1}^k\left|\pi_i(P_i - \beta)\right|\right)^2 = \frac{wJ(h)^2}{8\beta(1 - \beta)} \nonumber\\
&\geq& \frac{J(h)^2G_t^e}{8\beta(1 - \beta)(t+1)\ln k} \geq \frac{\gamma^2G_t^e}{2(1-\gamma)^2(t+1)\ln k}, 
\label{eq:whohoe}
\end{eqnarray}
where the last inequality comes from the fact that $1 - \gamma \geq \beta \geq \gamma$ (see the definition of $\gamma$ in the \emph{weak hypothesis assumption}) and $J(h) \geq 2\gamma$ (see \emph{weak hypothesis assumption}). For the Gini-entropy criterion notice that from Equation~\ref{eqn:gent-decrease}, Lemma~\ref{lem:sc_gentropy} and Lemma~\ref{lem:norms}, it follows that
\begin{eqnarray}
\Delta_t^g &\geq& w\beta(1\!-\!\beta)\|\bm \pi_0 \!-\! \bm \pi_1\|_2^2 \geq \frac{1}{k}w\beta(1\!-\!\beta)\|\bm \pi_0 \!-\! \bm \pi_1\|_1^2 \geq \frac{\gamma^2G_t^g}{(1 \!-\! \gamma)^2(t\!+\!1)(k\!-\!1)},
\label{eq:whohoge}
\end{eqnarray}
where the last inequality is obtained similarly as the last inequality in Equation~\ref{eq:whohoe}.
And finally for the modified Gini-entropy it follows from Equation~\ref{eqn:mgent-decrease}, Lemma~\ref{lem:sc_mgentropy} and Lemma~\ref{lem:norms} that
\begin{eqnarray}
\Delta_t^m &\geq& w\frac{(\mathcal{C}-2)^2}{\mathcal{C}^3}\beta(1-\beta)\|\bm \pi_0 - \bm \pi_1\|_2^2 \geq \frac{1}{k}w\frac{(\mathcal{C}-2)^2}{\mathcal{C}^3}\beta(1-\beta)\|\bm \pi_0 - \bm \pi_1\|_1^2 \nonumber\\
&\geq& \frac{\gamma^2G_t^m}{\frac{\mathcal{C}^3}{(\mathcal{C}-2)^2}(1 - \gamma)^2(t+1)k\sqrt{k\mathcal{C}-1}},
\label{eq:whohoge}
\end{eqnarray}
where the last inequality is obtained as before.

Clearly the larger the objective $J(h)$ is at time $t$, the larger the entropy reduction ends up being,
which confirms the plausibility of the approach in~\cite{DBLP:journals/corr/ChoromanskaL14} where the goal is to maximize $J(h)$. Let 
\begin{equation}
\eta^e = \frac{2\sqrt{2}\gamma}{(1-\gamma)\sqrt{\ln k}},\:\:\:\:\:\eta^g = \frac{4\gamma}{(1-\gamma)\sqrt{k-1}},\:\:\:\:\:\eta^m = \frac{4\gamma}{(1-\gamma)\sqrt{\frac{\mathcal{C}^3}{(\mathcal{C}-2)^2}k\sqrt{k\mathcal{C}-1}}}.
\label{eq:etas}
\end{equation}
For simplicity of notation assume $\Delta_t$ corresponds to either $\Delta_t^e$, or $\Delta_t^g$, or $\Delta_t^m$, and $G_t$ stands for $G_t^e$, or $G_t^g$, or $G_t^m$. Thus $\Delta_t > \frac{\eta^2G_t}{16(t+1)}$, and we obtain the recurrence inequality
\[G_{t+1} \leq G_t - \Delta_t < G_t - \frac{\eta^2G_t}{16(t+1)} =
G_t\left(1 - \frac{\eta^2}{16(t+1)}\right)
\]
One can now compute the minimum number of splits required to reduce
$G_t$ below $\alpha$, where $\alpha \in [0,1]$. Assume $\log_2(t+1) \in \mathbb{Z}^{+}$.
\begin{eqnarray*}
&&G_{t+1} \leq G_t\left(1 - \frac{\eta^2}{16(t+1)}\right) = G_1\left(1 \!-\! \frac{\eta^2}{16\cdot 2}\right)\left(1 \!-\! \frac{\eta^2}{16\cdot 3}\right)\dots\left(1 \!-\! \frac{\eta^2}{16\cdot (t+1)}\right)\\
&&= G_1\!\!\left(1 \!-\! \frac{\eta^2}{16\cdot 2}\right)\!\!\!\prod_{t^{'}=3}^4\!\!\!\left(1 \!-\! \frac{\eta^2}{16\cdot t^{'}}\right)\!\dots \!\!\!\!\!\!\!\!\!\!\prod_{t^{'}=(2^r/2)+1}^{2^r}\!\!\!\!\!\left(1 \!-\! \frac{\eta^2}{16\cdot t^{'}}\right)\!\dots \!\!\!\!\!\!\!\!\!\!\prod_{t^{'}=(2^{\log_2(t+1)}/2)+1}^{2^{\log_2(t+1)}}\!\!\!\!\!\left(1 \!-\! \frac{\eta^2}{16\cdot t^{'}}\right),
\end{eqnarray*}
where $r = \{2,3,\dots,\log_2(t+1)\}$. Recall that
\[\prod_{t^{'}=(2^r/2)+1}^{2^r}\!\!\!\left(1 - \frac{\eta^2}{16\cdot t^{'}}\right) \leq \!\!\!\prod_{t^{'}=(2^r/2)+1}^{2^r}\!\!\!\left(1 - \frac{\eta^2}{16\cdot 2^r}\right) = \left(1 - \frac{\eta^2}{16\cdot 2^r}\right)^{2^r/2} \leq e^{-\eta^2/32},
\]
where the last step follows from Lemma~\ref{lem:exp}. Also note that by the same lemma $\left(1 - \frac{\eta^2}{16\cdot 2}\right) \leq e^{-\eta^2/32}$.
Thus
\begin{equation}
G_{t+1} \leq G_1e^{-\eta^2\log_2(t+1)/32}.
\label{eq:G}
\end{equation}
Therefore to reduce $G_{t+1} \leq \alpha$ (where $\alpha$'s are defined in Theorems~\ref{thm:main1},~\ref{thm:main2}, and~\ref{thm:main3}) it suffices to make $t+1$ splits such that $\log_2 (t+1) \geq \ln\left(\frac{G_1}{\alpha}\right)^{\frac{32}{\eta^2}}$ splits.  
Since $\log_2 (t+1) = \ln (t+1)\cdot\log_2(e)$, where $e = \exp(1)$. Thus
\begin{equation}
\ln (t+1) \geq \ln\left(\frac{G_1}{\alpha}\right)^{\frac{32}{\eta^2\log_2(e)}} \Leftrightarrow t+1 \geq \left(\frac{G_1}{\alpha}\right)^{\frac{32}{\eta^2\log_2(e)}}.
\label{eq:tp1}
\end{equation}
Recall that by resp. Lemma~\ref{lem:bound_entropy},~\ref{lem:bound_gentropy}, and ~\ref{lem:bound_mgentropy} we have resp. $G_1^e \leq 2\ln k$, $G_1^g \leq 2(1-\frac{1}{k})$, $G_1^g \leq 2\sqrt{k\mathcal{C}-1}$.
We consider the worst case setting (giving the largest possible number of split) thus we assume $G_1^e = 2\ln k$, $G_1^g = 2(1-\frac{1}{k})$, and $G_1^g \leq 2\sqrt{k\mathcal{C}-1}$.
Combining that with Equation~\ref{eq:etas} and Equation~\ref{eq:tp1} yields statements of the main theorems.
\end{proof}

\section{Numerical experiments}
\label{sec:exp}

We run \textit{LOMtree} algorithm, which is implemented in the open source learning system Vowpal Wabbit~\cite{VowpalWabbit}, on four benchmark multiclass datasets: \textit{Mnist} ($10$ classes, downloaded from \url{http://yann.lecun.com/exdb/mnist/}), \textit{Isolet} ($26$ classes, downloaded from \url{http://www.cs.huji.ac.il/~shais/datasets/ClassificationDatasets.html}), \textit{Sector} ($105$ classes, downloaded from \url{http://www.csie.ntu.edu.tw/~cjlin/libsvmtools/datasets/multiclass.html}), and \textit{Aloi} ($1000$ classes, downloaded from \url{http://www.csie.ntu.edu.tw/~cjlin/libsvmtools/datasets/multiclass.html}). The datasets were divided into training ($90\%$) and testing ($10\%$), where $10\%$ of the training dataset was used as a validation set. The regressors in the tree nodes are linear and were trained by SGD~\cite{bottou-98x} with $20$ epochs and the learning rate chosen from the set $\{0.25,0.5,0.75,1,2,4,8\}$. We investigated different swap resistances\footnote{see~\cite{DBLP:journals/corr/ChoromanskaL14} for details} chosen from the set $\{4,8,16,32,64,128,256\}$. We selected the learning rate and the swap resistance as the one minimizing the validation error, where the number of splits in all experiments were set to $10k$. 
\begin{figure}[h]
\center
\setlength\tabcolsep{0pt}
\begin{tabular}{cc}
\includegraphics[width = 0.5\textwidth]{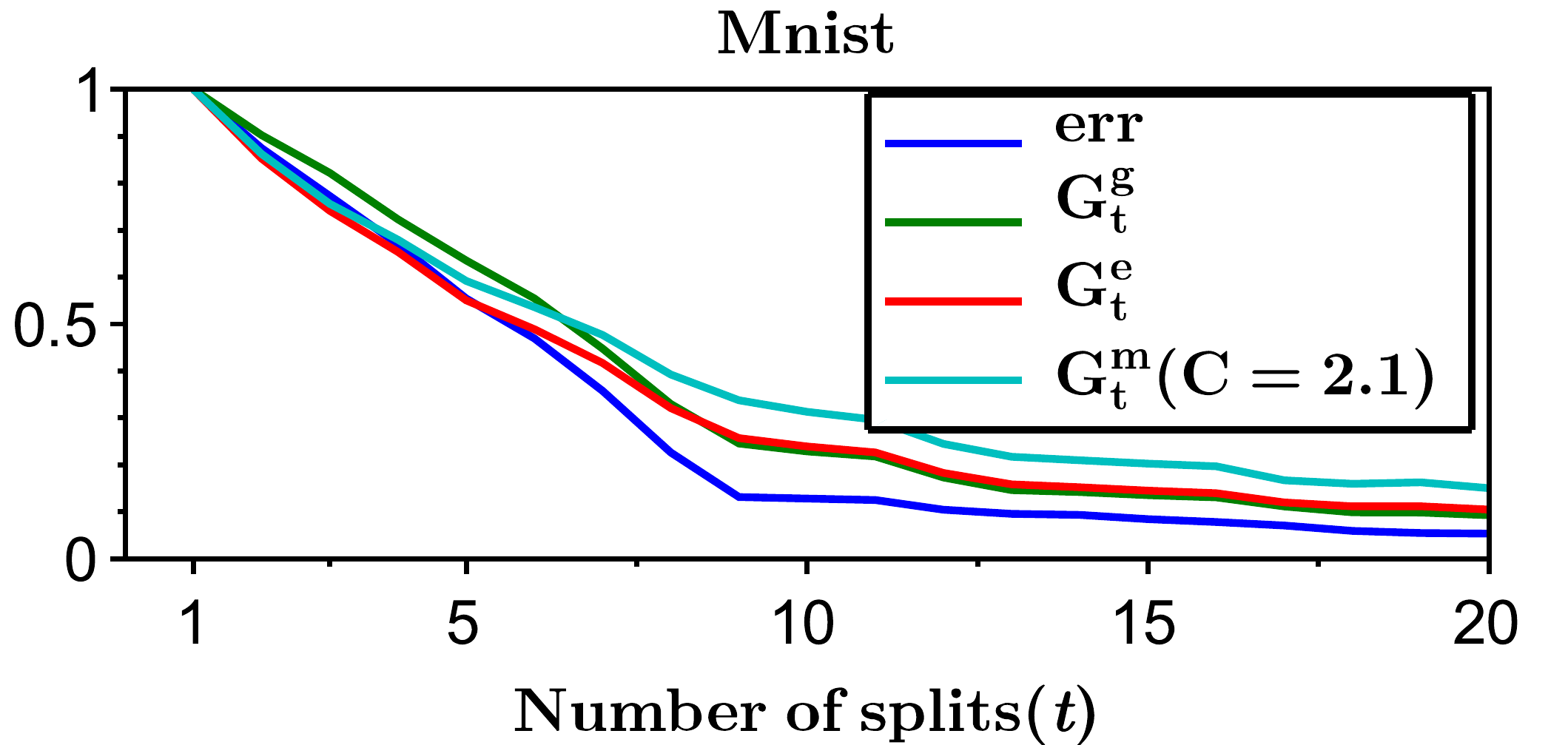}
\includegraphics[width = 0.5\textwidth]{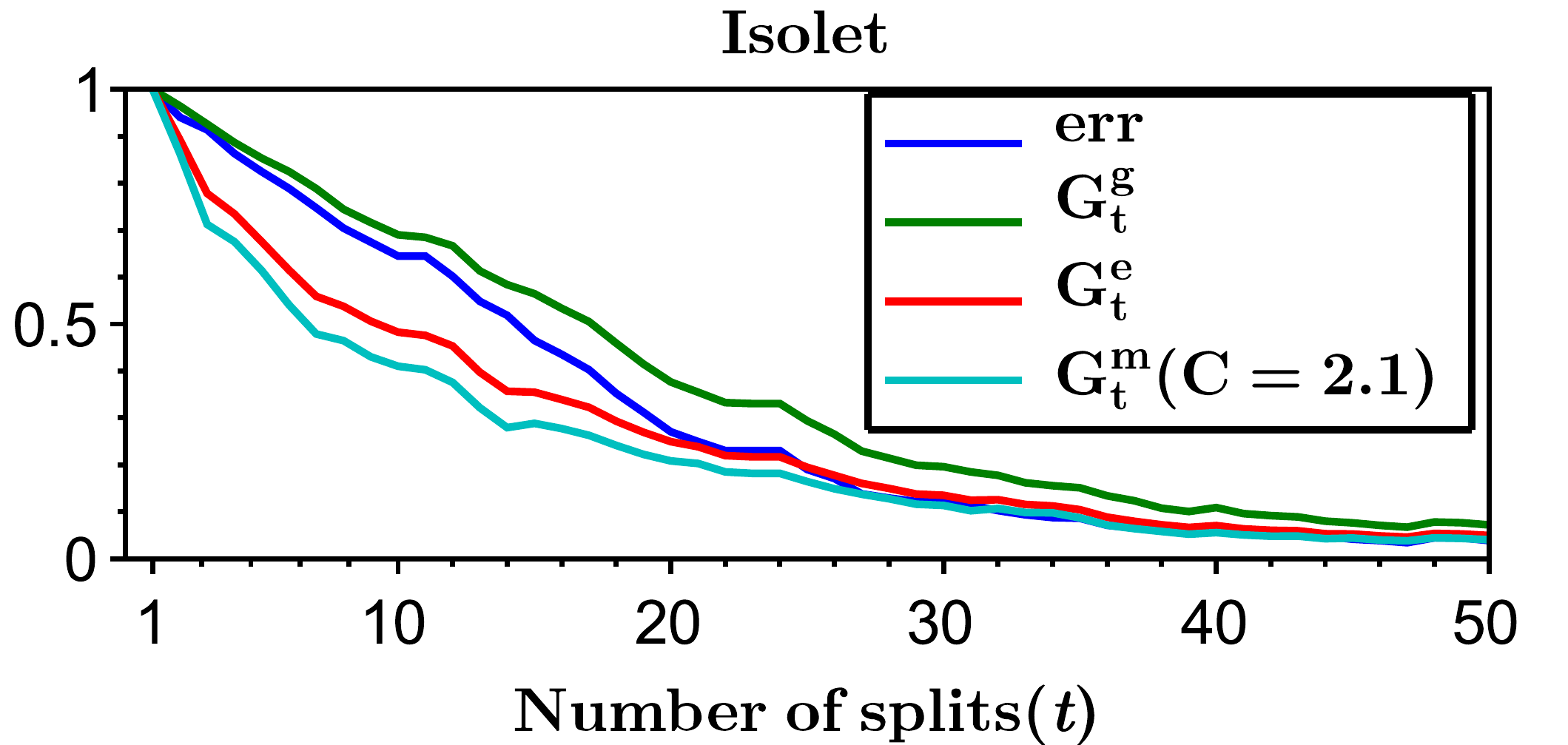}\\
\includegraphics[width = 0.5\textwidth]{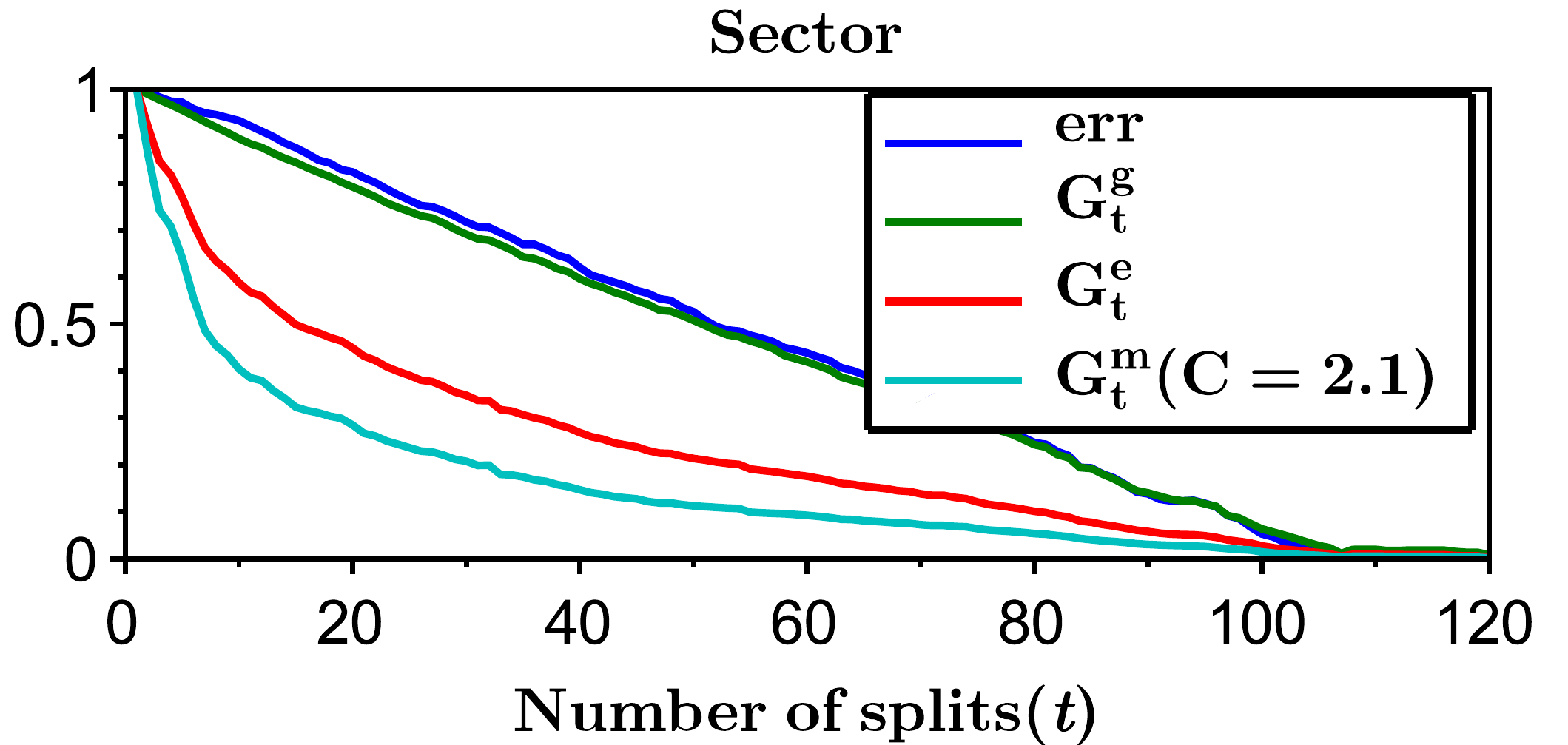}
\includegraphics[width = 0.5\textwidth]{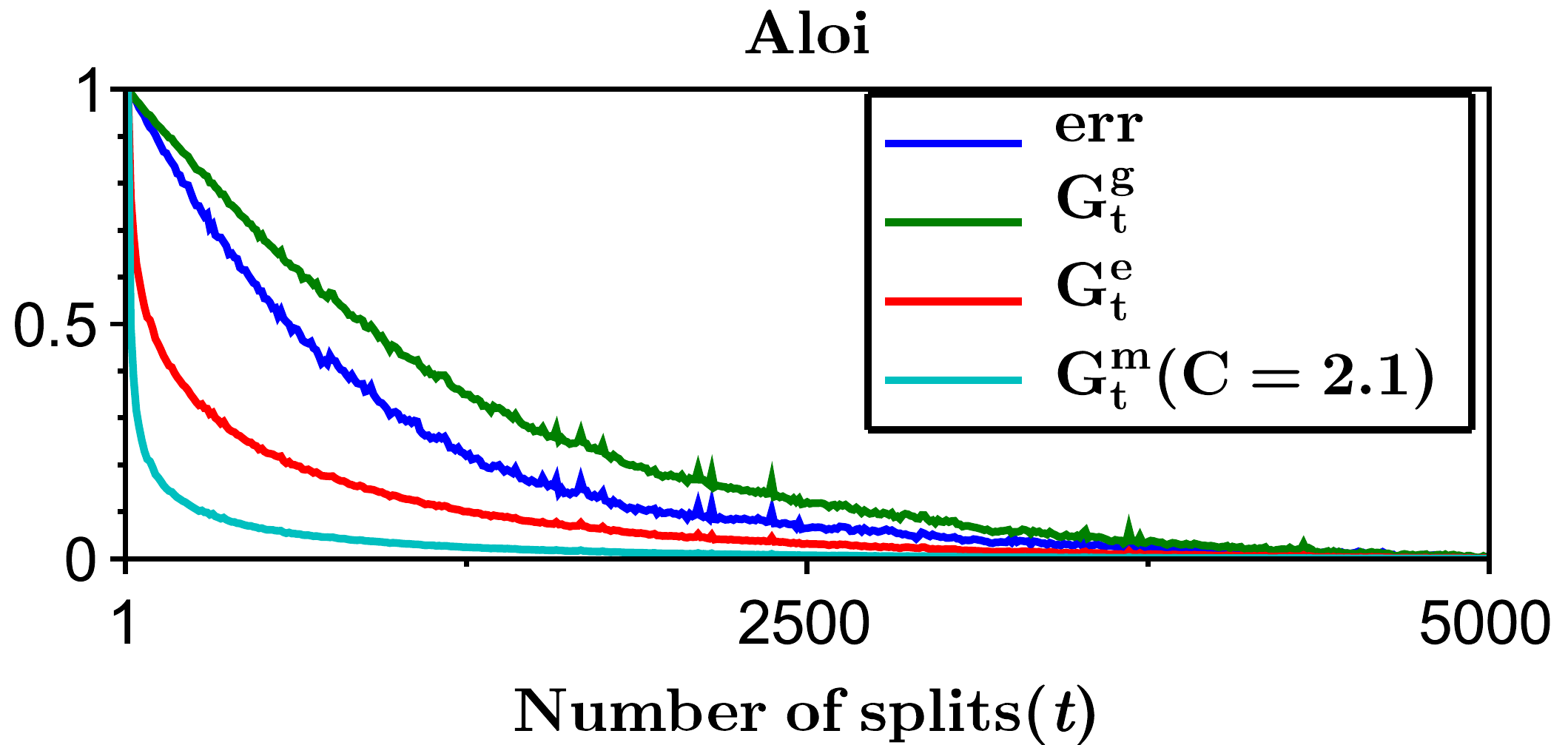}\\
\end{tabular}
\caption{Functions $G_t^e$, $G_t^e$, and $G_t^m$, and the test error, all normalized to the interval $[0,1]$, versus the number of splits. \textit{Figure is recommended to be read in color.}}
\label{fig:balance_and_purity}
\label{fig:curves}
\end{figure}

Figure~\ref{fig:curves} shows the entropy, Gini-entropy, modified Gini-entropy, and the error, all normalized to the interval $[0,1]$, as the function of the number of splits. The behavior of the entropy and Gini-entropy match the theoretical findings. However, the modified Gini-entropy instead drops the fastest with the number of splits, which in particular suggests that in this case perhaps tighter bounds could possibly be proved (for the binary case tighter analysis was shown in~\cite{Kearns95}, but it is highly non-trivial to generalize this analysis to the multiclass classification setting). Furthermore, it can be observed that the behavior of the error closely mimics the behavior of the Gini-entropy.

\section{Conclusions}
\label{sec:con}

This paper focuses on the properties of the recently proposed LOMtree algorithm. We provide an exhaustive theoretical analysis of the objective function underlying the algorithm. We show a unified framework for analyzing the boosting ability of the algorithm by exploring the connection of its objective to entropy-based criteria, such as entropy, Gini-entropy and its modified version. We show that the strong concativity properties of these criteria have critical impact on the character of the obtained bounds. The experiments suggest that perhaps tighter bound is possible in particular for the modified version of the Gini-entropy.

\vskip 0.2in
\bibliography{Choromanska15a}

\end{document}